\pdfoutput=1
\documentclass[11pt]{article}
\usepackage[whole]{bxcjkjatype}
\usepackage[preprint]{acl}
\usepackage{times}
\usepackage{latexsym}
\usepackage[T1]{fontenc}
\usepackage[utf8]{inputenc}
\usepackage{microtype}
\usepackage{inconsolata}
\usepackage{microtype}
\usepackage{comment}
\usepackage{graphicx}
\usepackage{amsmath}
\usepackage{amssymb}
\usepackage{amsthm}
\usepackage{amsfonts}
\usepackage{array}
\usepackage{multirow}
\usepackage{booktabs}
\usepackage{adjustbox}
\usepackage{subcaption}
\usepackage{stmaryrd}
\usepackage{siunitx}
\usepackage{bm}
\newcommand{\bbR}{\mathbb{R}}

\newcommand{\bmq}{\bm{q}}
\newcommand{\bmQ}{\bm{Q}}
\newcommand{\bmL}{\bm{L}}

\DeclareMathOperator*{\Var}{Var}

\DeclareMathOperator*{\bbE}{\mathbb{E}}
\newcommand{\KL}{\mathrm{KL}}
\newcommand{\unif}{\mathrm{unif}}
\newcommand{\LS}{\mathrm{LS}}
\theoremstyle{plain}

\newtheorem*{thm*}{Theorem}
\theoremstyle{definition}

\newtheorem{lemma}{Lemma}
\newtheorem{proposition}{Proposition}
\newcommand{\wtildeQ}{\widetilde{\bmQ}}
\newcommand{\wtildeq}{\widetilde{\bmq}}

\usepackage{longtable}

\title{Likelihood Variance as Text Importance \\ for Resampling Texts to Map Language Models}

\author{
Momose Oyama${}^{1,2}$\quad Ryo Kishino${}^{1}$\quad Hiroaki Yamagiwa${}^{1}$\quad Hidetoshi Shimodaira${}^{1,2}$\\
${}^{1}$Kyoto University\quad
${}^{2}$RIKEN\\
\texttt{oyama.momose@sys.i.kyoto-u.ac.jp, kishino.ryo.32s@st.kyoto-u.ac.jp}\\
\texttt{\{h.yamagiwa, shimo\}@i.kyoto-u.ac.jp}
}

\begin{document}
\maketitle
\begin{abstract}
We address the computational cost of constructing a model map, which embeds diverse language models into a common space for comparison via KL divergence. The map relies on log-likelihoods over a large text set, making the cost proportional to the number of texts. To reduce this cost, we propose a resampling method that selects important texts with weights proportional to the variance of log-likelihoods across models for each text. Our method significantly reduces the number of required texts while preserving the accuracy of KL divergence estimates. Experiments show that it achieves comparable performance to uniform sampling with about half as many texts, and also facilitates efficient incorporation of new models into an existing map. These results enable scalable and efficient construction of language model maps.
\end{abstract}

\section{Introduction}
For the systematic comparison of numerous language models, \citet{modelmap2025} proposed a method to estimate Kullback-Leibler (KL) divergence between language models based on log-likelihood vectors. In this method, language models with different architectures are embedded into a common space, and visualizing this creates a map of language models (Fig.~\ref{fig:modelmap-error-ellipse}). 

The model map is constructed using a text set $D_N = \{x_1, \cdots, x_N\}$. Since $D_N$ is a sample from a broader text population $D^{\dagger}=\{x_1^\dagger,\ldots,x_{N_0}^\dagger\}$, it
introduces sampling error relative to the true relationships between models in the population\footnote{When the entire Pile corpus is converted into 1,024-byte text chunks, the total number of texts is $N_0=$ 5,703,791.}. Additionally, the computational cost, which is proportional to the number of texts $N$, is also an issue.

\begin{figure}[t]
    \centering
    \includegraphics[width=\linewidth]{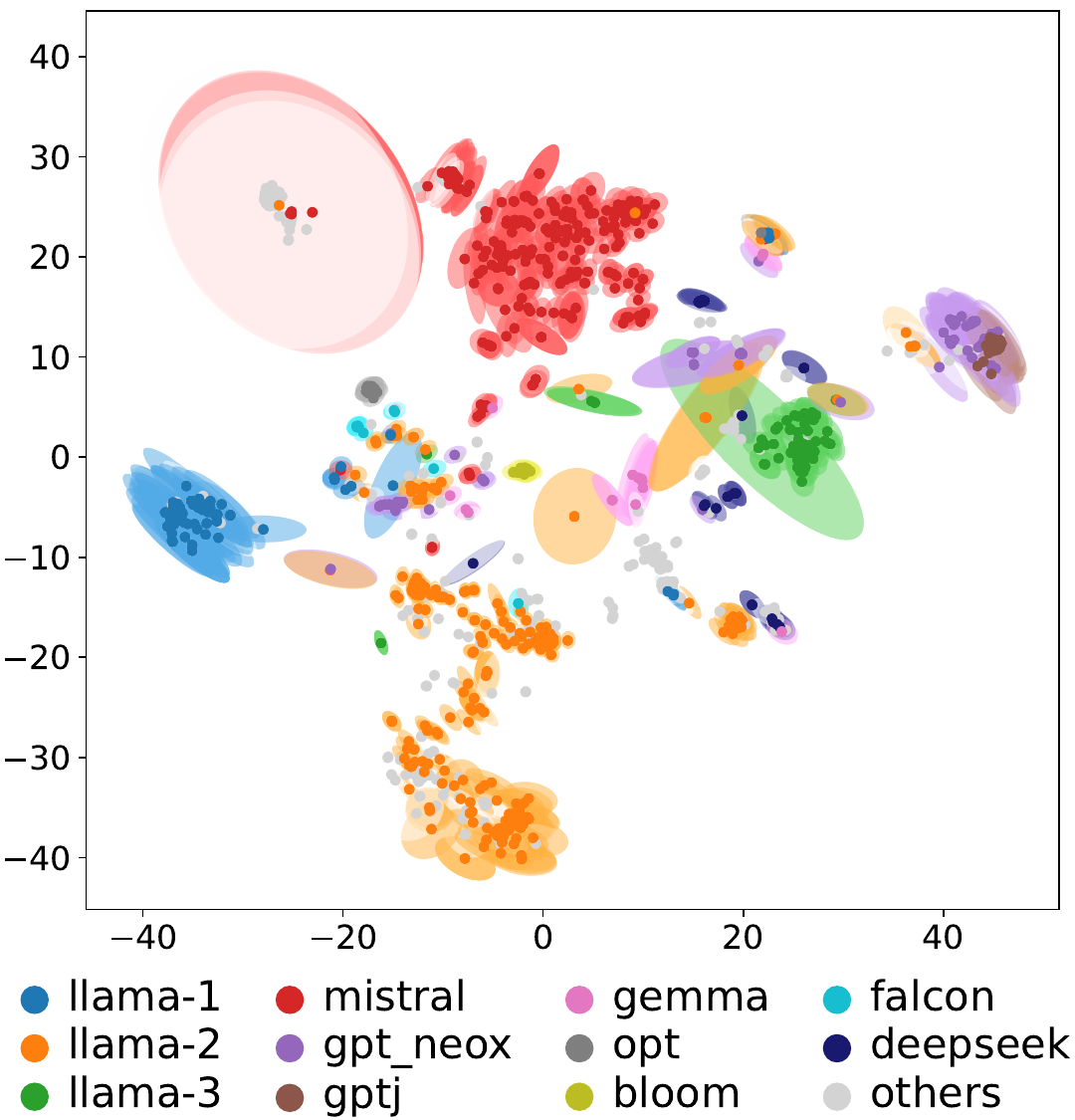}
    \caption{The model map calculated with $D_N$ is visualized using t-SNE. The sampling error for each point was estimated using the bootstrap resampling method and shown as an ellipse. See Appendix~\ref{sec:app-model-map} for details.}
    \label{fig:modelmap-error-ellipse}
\end{figure}

To reduce computational costs for adding new models to the map, we consider resampling $n$ texts from the data $D_N$ with replacement, and then reconstruct the model map using the resulting set $D^{\ast}_d$ of $d$ unique texts ($d\le n$). 
When we discuss the distance between models calculated using log-likelihood vectors and the reliability of the model map, it is crucial to focus on the sampling error with respect to the population.

To estimate the resampling error of the distance between models measured based on $D^{\ast}_d$ relative to the true distances in the population, two types of errors need to be considered. 
The first is the \textbf{sampling error} that the original data $D_N$ has with respect to the population $D^{\dagger}$, and the second is the \textbf{resampling error} incurred when selecting $D^{\ast}_d$ from $D_N$. 

Since uniform resampling from data $D_N$ has limitations, we propose a method to resample important texts with high probability. 
This idea is based on the column sampling method of~\citet{drineas-kannan-2001}, which was originally proposed for approximating matrix products.
Through experiments, we show that our proposed text resampling method achieves comparable estimation error to uniform random sampling with about half the number of texts, and is also effective when adding new models to an existing model map. 

\section{Background}
\subsection{Map of Language Models} \label{sec:model-map}
The map of $K$ language models $p_1, \dots, p_K$ is constructed based on a text set $D_N = \{x_1, \cdots, x_N\}$.
The log-likelihood of model $p_i$ for text $x_s$ is denoted by $\log p_i (x_s)$, and the log-likelihood matrix $\bmL\in \bbR^{K\times N}$ is composed of these elements.
Let $\bmQ\in \bbR^{K\times N}$ be the matrix obtained by applying double centering\footnote{Centering is performed row-wise (per model) and then column-wise (per text).} to matrix $\bmL$.
The $i$-th row vector $\bmq_i \in \bbR^{N}$ of this matrix serves as the coordinate for language model $p_i$. 
\citet{modelmap2025} showed that the KL divergence between models $p_i$ and $p_j$ can be estimated by the following equation:
\begin{align*}
2\,\mathrm{KL}(p_i,p_j) \approx \|\bmq_i - \bmq_j\|^2/N.
\label{eq:estimated-KL}
\end{align*}

\subsection{Length Squared Sampling} \label{sec:length-squared-sampling}
The idea of resampling text based on its importance is derived from column sampling methods, which probabilistically select a small number of columns from a matrix $\bm{A} = (A^{(1)},\ldots,A^{(N)}) \in \bbR^{K \times N}$ to approximate the matrix product $\bm{A}\bm{A}^{\top}$.

In the representative Length Squared (LS) sampling method~\cite{drineas-kannan-2001}, each column $A^{(s)}$ is sampled with a probability proportional to the square of its Euclidean norm, $ \| A^{(s)} \|^{2}$. This is known to minimize the expected Frobenius norm of the approximation error for $\bm{A}\bm{A}^{\top}$.

\section{Resampling Texts for Model Map} \label{sec:N-to-d}

\subsection{Text Resampling Method}
We apply the idea of LS sampling to reduce the number of texts used for the model map, while estimating the model distance $\|\bmq_i - \bmq_j\|^2$ as accurately as possible. 
To determine the probability $\pi_s$ that text $x_s$ is resampled from dataset $D_N = \{x_1, \dots, x_N\}$, we utilize the information in the double centered log-likelihood matrix $\bmQ \in \bbR^{K\times N}$.

\paragraph{LS Sampling.}
Following \citet{drineas-kannan-2001}, we set $\pi_s \propto \|Q^{(s)}\|^2$.
The squared norm $\|Q^{(s)}\|^2$ corresponds to the variance of the log-likelihoods for text $x_s$ across models.
This is optimal for approximating the inner product $\bmq_i^{\top}\bmq_j$, but not necessarily optimal for our goal of approximating the squared Euclidean distance $\| \bmq_i - \bmq_j\|^2$.

\paragraph{KL Sampling.}
Therefore, we aim to directly minimize the estimation error of the squared Euclidean distance and propose KL sampling, where
\[
\pi_s \propto \sqrt{\sum_{i, j=1}^K (q_i{(x_s)}-q_j{(x_s)})^4}.
\]
Here, $q_i{(x_s)}$ represents the $(i,s)$-th component of matrix $\bmQ \in \bbR^{K \times N}$. 
We show in Appendix~\ref{sec:kl-sampling} that this method yields an optimal approximation of the squared Euclidean distance $\| \bmq_i - \bmq_j\|^2$. 

\paragraph{Baseline: Uniform Sampling.}
All texts are resampled with equal probability $\pi_s = 1/N$.

\subsection{Model Map with Resampled Texts} \label{sec:model-map-resampled-texts}
Let $\widetilde{D}_n$ be the dataset obtained by resampling $n$ texts from $D_N$, and the set of $d$ unique texts in $\widetilde{D}_n$ be $D_d^{\ast} = \{x_{u_1}, \dots, x_{u_d}\}$. 
We denote $c(u_t)$ as the number of times each text $x_{u_t}$ was resampled, such that $\sum_{t=1}^d c(u_t) = n$.

We construct the log-likelihood matrix $\bmL_d \in \bbR^{K \times d}$ using the $d$ resampled texts and obtain $\wtildeQ_d$ by double centering.
In row-wise centering, the column $L^{(u_t)}$ corresponding to the resampled text $x_{u_t}$ is weighted by $w_{u_t} = c(u_t)/n\pi_{u_t}$, based on the resampling probability $\pi_s$ and the number of times it was resampled $c(u_t)$.
Let $\wtildeq_i\in\bbR^d$ be the row vector of $\wtildeQ_d$. The distance of model $p_i$ and $p_j$ is calculated by the squared Euclidean distance using weights $\bm w_d=(w_{u_1},\ldots,w_{u_d})$:
\[
\| \wtildeq_i - \wtildeq_j \|_{\bm{w}_d}^2 := \sum_{t=1}^{d} \frac{c(u_t)}{n\pi_{u_t}} \left(\widetilde{q}_i (x_{u_t}) - \widetilde{q}_j (x_{u_t})\right)^2.
\]
Here, $\widetilde{q}_i (x_{u_t})$ is the $(i, t)$-th component of matrix $\wtildeQ_d\in \bbR^{K \times d}$.

\section{Experiments} \label{sec:resampling-experiment}

\subsection{Error of Resampled Estimates} \label{sec:resampling-error}
The dataset $D_N$ is a random sample from the text population $D^\dagger$, and $\widetilde{D}_n$ is resampled dataset from $D_N$.
The squared Euclidean distances computed from $D^\dagger$, $D_N$, and $\widetilde D_n$ are, respectively, $g_{ij}^\dagger = (N/N_0)\| \bmq_i^\dagger - \bmq_j^\dagger\|^2$, $g_{ij}=\|\bmq_i - \bmq_j\|^2$, $\widetilde g_{ij}=\| \wtildeq_i - \wtildeq_j \|_{\bm{w}_d}^2$ 
where the scaling ensures comparability across different dimensionalities~\footnote{The coordinate of $p_i$ in $D^\dagger$ is denoted by $\bmq_i^\dagger\in \bbR^{N_0}$. $g_{ij}^\dagger$ and $\widetilde g_{ij}$ are scaled to match $g_{ij}$.}.

Our objective is to evaluate the error, which is how much $\widetilde{g}_{ij}$ deviates from the unknown true value $g^\dagger_{ij}$ in the population.
Since $N_0\gg N$, $g_{ij}^\dagger$ is difficult to compute, so we estimate the error using $D_N$.
Details on the error estimation method are described in Appendix~\ref{sec:resampling-error-to-population}.

\paragraph{Resampling Error to Population ($\hat \sigma_{\mathrm{Method}, n}$).}
Let $\hat \sigma_{\mathrm{Method}, n}$ ($\mathrm{Method} \in \{\LS, \KL, \unif\}$) be the estimated error relative to the true value $g^\dagger_{ij}$ in the population $D^\dagger$.
This is estimated by considering the following two errors: 
\[
\hat\sigma_{\mathrm{Method}, n} = \sqrt{\tau_{\unif, N}^2 +  \tau_{\mathrm{Method}, n}^2}.
\]
Here, $\tau_{\unif, N}$ is the bootstrap estimate~\citep{efron-tibshirani-bootstrap} of the sampling error that $D_N$ itself has with respect to the population $D^\dagger$.
$\tau_{\mathrm{Method}, n}  = \sqrt{\tau_{\mathrm{Method}, n}^2}$ is the Root Mean Squared Error (RMSE) of resampling ($\mathrm{Method}$) with respect to $D_N$, and it is calculated from the aggregated MSE of the relative error $\widetilde{e}_{ij}={(\widetilde{g}_{ij} - g_{ij})}/{\max(g_{ij}, \varepsilon_0)}$ as $\tau_{n}^2 = \frac{1}{K^2R}\sum_{i,j=1}^{K} \sum_{r=1}^{R} \left(\widetilde{e}_{ij}^{(r)}\right)^2$.
In the experiments, we set $\varepsilon_0={10^{-3}}$, $R=100$, and varied $n$ from $10$ to $10{,}000$.

\paragraph{Baseline: Sampling Error ($\hat \kappa_m$).}
Let $\kappa_m$ be the aggregated relative error of $g_{ij}$ with respect to the true value $g^\dagger_{ij}$ in the population.
In this sampling, all $m$ texts are unique, so the number of unique texts is $d=m$.
The bootstrap estimate~\citep{efron-tibshirani-bootstrap} of $\kappa_m$ is $\hat \kappa_m = \tau_{\unif, m}$.

\begin{figure}[t]
    \centering
    \includegraphics[width=\linewidth]{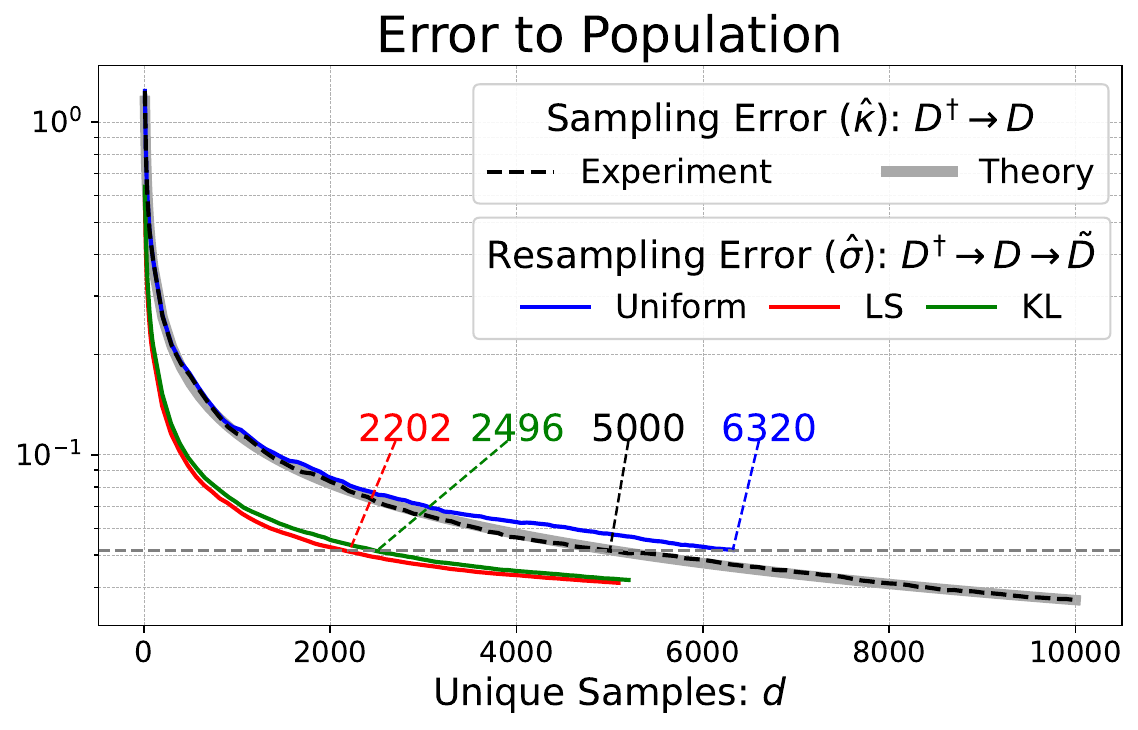}
    \caption{The number of unique texts $d$ for each resampling method and the estimation error against the population. `Sampling Error' corresponds to the estimation error when $d$ texts are randomly sampled without replacement from the population $D^{\dagger}$. Here, the error is normalized by KL divergence for each model pair. For error evaluation without normalization, see Appendix~\ref{sec:absolute-error}.}
    \label{fig:diff-ratio-distance}
\end{figure}

\paragraph{Results.}
Figure~\ref{fig:diff-ratio-distance} shows the relationship between the number of unique texts $d$ and the estimated error with respect to the population for each resampling method.
The dotted line represents the estimated error $\hat{\kappa}_d$ when $d$ texts are directly randomly sampled from the population $D^\dagger$, and the thick gray solid line shows its theoretical value, computed as $\sqrt{N/d} \,\hat{\kappa}_N$.  
The colored solid lines represent the estimated error $\hat{\sigma}_{\text{Method},n}$ when resampling from $D_N$ using each method.
Comparing the blue solid line (Uniform resampling from $D_N$) and the black dotted line (direct sampling from $D^\dagger$), Uniform resampling from $D_N$ results in an error comparable to directly sampling the same number of texts from the population\footnote{In Uniform resampling, weighting according to the number of duplicates slightly degrades performance.}.
In contrast, LS sampling and KL sampling achieve an error comparable to the estimated error $\hat \kappa_d$ of random sampling $d$ texts from the population with a smaller number of unique texts.
As can be seen from Table~\ref{tab:n-and-d}, an error comparable to $\hat \kappa_{5{,}000}$ is achieved with LS sampling using an average of \(d=2{,}202\) texts, and with KL sampling using an average of \(d=2{,}496\) texts, which is about half the number of texts required by uniform sampling from the population to achieve similar error.
These results suggest that LS sampling and KL sampling can achieve comparable estimation accuracy with about half the number of unique texts by selecting important texts.

\begin{table}[t]
\centering
\begin{adjustbox}{width=\linewidth}
\begin{tabular}{lrrrrrrrrrrrr}
\toprule
              $m$ & $500$ & $2{,}500$ & $5{,}000$ & $7{,}500$ \\
\midrule
Uniform        & $582\pm4$ & $2{,}877\pm18$ & N/A & N/A \\
KL             & $196\pm2$ & $1{,}070\pm10$ & $2{,}496\pm22$ & $4{,}780\pm33$ \\
LS             & $195\pm2$ & $895\pm8$ & $2{,}202\pm19$ & $4{,}229\pm30$ \\
\bottomrule
\end{tabular}
 \end{adjustbox}
\caption{Average and standard deviation of the number of unique texts $d$ in resampled dataset $\widetilde{D}_n$, where $n$ is the smallest value satisfying $\hat \sigma_{\mathrm{Method}, n} \le \hat \kappa_m$.}
\label{tab:n-and-d}
\end{table}

\subsection{Stability of Model Map under Resampling}
\label{sec:d-dim-modelmap}
\begin{figure*}[t]
  \centering
  \includegraphics[width=\linewidth]{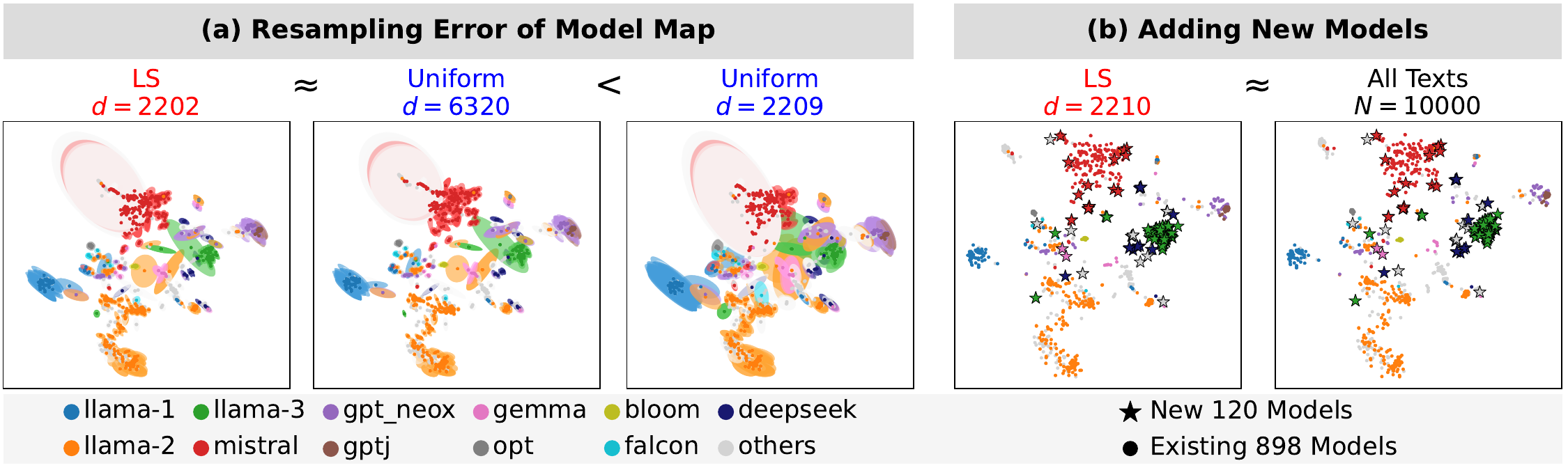}
    \caption{\textbf{(a)}
Model maps based on LS ($n=2900, d = 2202$), uniform ($n=10000, d = 6320$), and uniform ($n=2500, d = 2209$) sampling.  
Each map shows the mean coordinates and their variability as ellipses across 100 trials.
\textbf{(b)}~Model maps with 120 new models added to existing 898 models. 
The left panel uses $d = 2210$ unique texts selected by LS sampling with $n = 2900$.  
The right panel uses all $N = 10{,}000$ texts.
}
  \label{fig:tsne_N-d}
\end{figure*}

For both analyses below, we resample $n$ texts with replacement from $D_N$, resulting in $d$ unique texts.  
Model coordinates are computed from log-likelihoods over the sampled texts and visualized using t-SNE.  
See Appendix~\ref{sec:app-model-map} for details.

\paragraph{(a) Resampling Error of Model Map.}
To evaluate variability due to resampling, we repeated the process 100 times per setting.  
Here, $d$ denotes the average number of unique texts across trials.  
Figure~\ref{fig:tsne_N-d}(a) shows three model maps; ellipses indicate the standard deviation of model positions.  
The first two maps show similar stability, indicating that LS achieves uniform-level robustness with fewer texts.  
When $d$ is matched, LS still yields lower variability than uniform sampling.

\paragraph{(b) Adding New Models.}
We tested whether a small set of texts selected by LS sampling is sufficient for placing new models.  
From 898 models created before April 10, 2024, we sampled $n = 2900$ texts via LS, resulting in $d = 2210$ unique texts\footnote{Model creation dates were obtained using the Hugging Face's API.}.  
The remaining 120 models were treated as new additions.  
This setting uses a single resampling trial, so $d$ is the actual number of unique texts.  
We computed log-likelihoods using only these texts and visualized the updated map.  
As shown in Fig.~\ref{fig:tsne_N-d}(b), the result closely matches the full map based on all $N = 10{,}000$ texts, indicating that reliable placement is achievable with a small subset.

\begin{figure}[t]
    \centering
    \includegraphics[width=\linewidth]{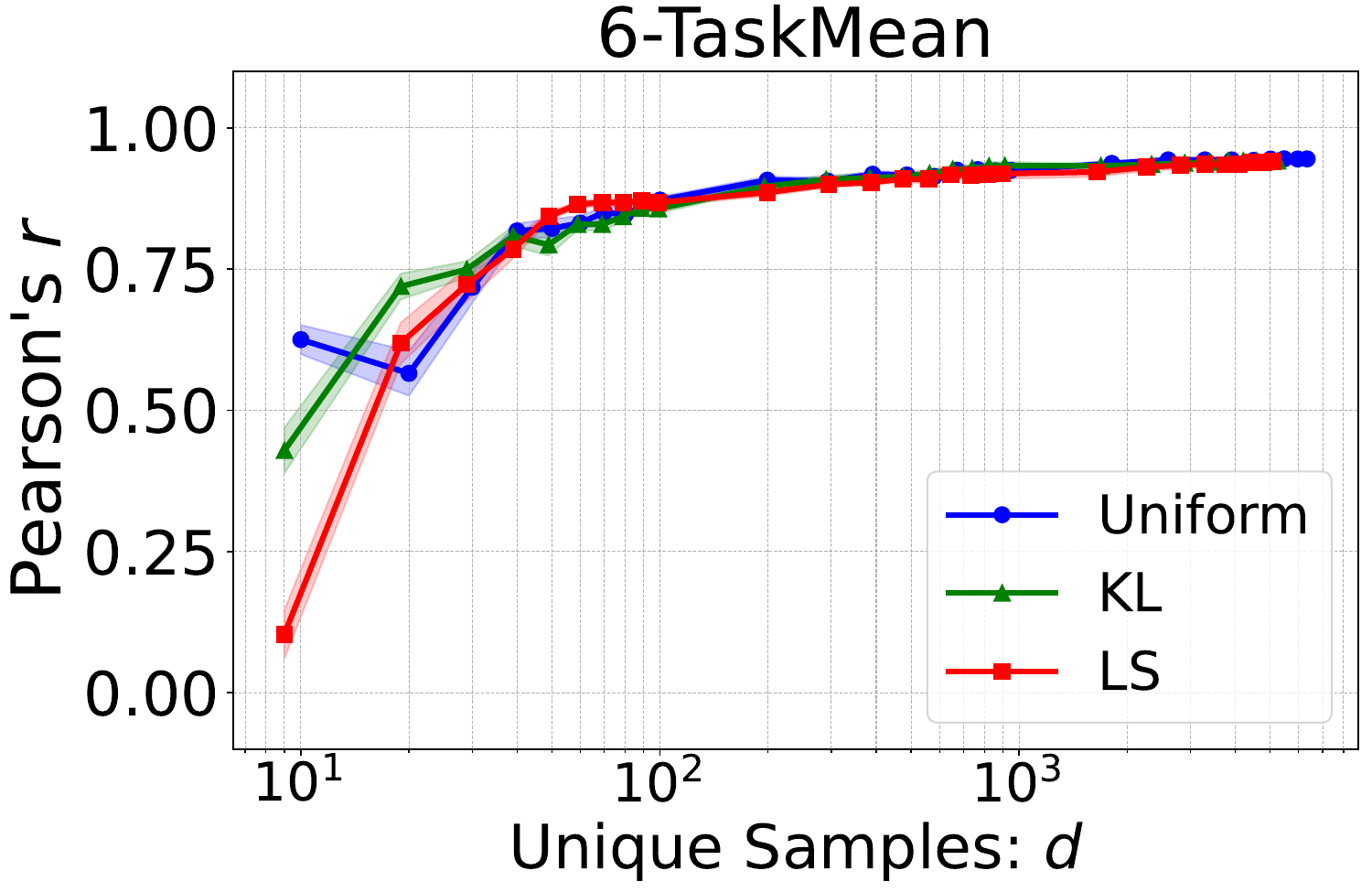}
    \caption{Pearson's correlation $r$ between predicted and actual scores on the average of six downstream tasks, shown as a function of the number of unique texts $d$.}
    \label{fig:pred_6taskmean}
\end{figure}

\subsection{Prediction of Downstream Performance} \label{sec:predict-performance-maintext}
Following \citet{modelmap2025}, we used model coordinates based on $d$ unique texts to predict the average performance on six downstream tasks (Fig.~\ref{fig:pred_6taskmean}).  
Although the resampling methods differ in KL estimation, prediction performance is similar across methods, suggesting that the resulting coordinates span comparable subspaces.  
At $d = 1000$, all methods achieve $r \approx 0.92$--$0.93$, indicating that a relatively small number of texts suffices for accurate prediction.  
See Appendix~\ref{sec:predict-performance} for details.

\section{Conclusion}
We discussed text selection methods to reduce the computational cost of calculating model maps.
Experimental results showed that, compared to uniform random sampling, our proposed methods can achieve comparable estimation error with approximately half the number of texts.
Furthermore, we confirmed that a small set of texts selected by this resampling is also effective when adding new models to an existing model map.
These results enable more efficient comparative analysis of large-scale language models. 

\clearpage
\section*{Limitations}

\begin{itemize}
    \item The aspect of data sampling remains unexplored. Future investigation is required to understand the extent to which the discussions in this study apply to datasets different from the $D_N$ used in this study.
    \item Regarding the experiments for adding new models, we did not include a detailed discussion on the number or types of the new models added.
   \item While LS and KL sampling both outperformed uniform sampling in KL divergence estimation, the difference disappeared in downstream performance prediction, where all methods performed similarly.  
This suggests that improvements in distance estimation do not necessarily lead to gains in downstream utility, and the relationship between the two warrants further investigation.
\end{itemize}

\section*{Acknowledgments}
This study was partially supported by JSPS KAKENHI 22H05106, 23H03355, JST CREST JPMJCR21N3, JST BOOST JPMJBS2407.

\bibliography{bib/anthology,bib/custom,bib/modelcard}

\appendix

\section{Author Contributions}
M.O. conceived the overall research idea of reducing computational cost via resampling, developed the LS sampling algorithm using matrix approximation techniques, and implemented all resampling methods and experiments related to model distance estimation and map construction.  
R.K. proposed the KL sampling method and proved its optimality.  
H.S. developed the statistical framework for analyzing sampling and resampling errors, and provided theoretical justification.  
H.Y. conducted the experiments on downstream task performance prediction.  
All authors contributed to writing the manuscript. M.O. coordinated the writing process and took the lead on the overall composition, while each author wrote the parts corresponding to their contributions.
The project was supervised by M.O. and H.S.

\section{Details of Model Map Construction} \label{sec:app-model-map}
\subsection{Log-Likelihood Data of Language Models} \label{sec:log-likelihood-data}
In our experiment, we used log-likelihood data~\cite{modelmap2025}\footnote{\url{https://github.com/shimo-lab/modelmap}} for \(K=1{,}018\) language models, calculated on \(N=10{,}000\) texts extracted from the Pile~\cite{arxiv:2101.00027}.

In the experiment described in Section~\ref{sec:d-dim-modelmap} for adding new models to the model map, we clipped the log-likelihood matrix $\bmL$ of the existing 898 models at the lower 2nd percentile value $-1495.9$, following \citet{modelmap2025}. The log-likelihood values of the newly added 120 models were also clipped using $-1495.9$ as the threshold.

\subsection{Visualization of Model Map} \label{sec:tsne-alignment}

This section details the procedure for generating the t-SNE~\cite{vanderMaaten-2008-tsne} visualizations and standard deviational ellipses presented in Fig.~\ref{fig:modelmap-error-ellipse} and Fig.~\ref{fig:tsne_N-d}.

\paragraph{Coordinate Computation.}
Resampling method and the number of texts to resample, $n$, are chosen.
Based on the unique resampled texts $D^{\ast}_d$, the model coordinates $\wtildeQ_d \in \bbR^{K \times d}$ are computed. These $d$ dimensional coordinates $\wtildeQ_d$ are subsequently reduced to two dimensions using the t-SNE algorithm.
This entire process, from resampling to t-SNE, is repeated $R=100$ times to assess variability. For consistency across these $R$ trials, the `random\_state` parameter of the t-SNE algorithm is fixed to 42. The initial coordinates for t-SNE in each trial are determined by applying PCA to the $\wtildeQ_d$ matrix calculated for that specific trial.

\paragraph{Coordinate Alignment.}
Let $\bm{X}_r \in \bbR^{K \times 2}$ denote the matrix of t-SNE coordinates obtained from the $r$-th trial (for $r=1,\dots, R$), and $\bm{X}_{\mathrm{ref}}$ as the t-SNE coordinates of $\bmQ$.
Since t-SNE results can vary due to inherent translational and rotational ambiguities, an alignment procedure is necessary to compare the $R$ sets of coordinates.
First, each set of coordinates $\bm{X}_r$ is centered by subtracting its mean: $\bm{Y}_r := \bm{X}_r-\bar{\bm{X}}_r$.
Next, the Orthogonal Procrustes~\cite{Schoenemann1966AGS} analysis is applied to align these centered coordinate sets.
For each subsequent set $\bm{Y}_r$ (where $r=1,\dots, R$), an orthogonal transformation matrix $\bm{U}_{r}$ is found 
that best aligns $\bm{Y}_r$ with $\bm{Y}_{\mathrm{ref}} = \bm{X}_{\mathrm{ref}} - \bar{\bm{X}}_{\mathrm{ref}}$. 
The aligned coordinates are then given by $\bm{Z}_r:=\bm{Y}_r\bm{U}_{r}$.

\paragraph{Centrography.}
After aligning all $R=100$ sets of t-SNE coordinates $\bm{Z}_r$, the mean coordinate $\bar{\bm{z}}_i \in \bbR^{2}$ and the covariance matrix $\mathrm{Cov}({\bm{z}}_i) \in \bbR^{2\times 2}$ are calculated for each model $i$ across the $R$ trials. The standard deviational ellipses~\cite{Yuill-1971-std-ellipse} shown in the figures are derived from these covariance matrices, with their height, width, and angle determined by the eigenvalues and eigenvectors of $\mathrm{Cov}({\bm{z}}_i)$.

\section{KL Sampling} \label{sec:kl-sampling}
\paragraph{Notation.}
We resample a text $x_s$ from the dataset $D_N = \{x_1, \ldots, x_N\}$ with probability $\pi_s$. Let the $n$ resampled texts be denoted by $\{x_{u_1}, \ldots, x_{u_n}\}$. Define $\widetilde{\bm Q} = (Q^{(u_1)}, \ldots, Q^{(u_n)}) \in \mathbb{R}^{K \times n}$, and $\widetilde{\bm q}_i \in \mathbb{R}^n$ as the $i$-th row vector of $\widetilde{\bm Q}$. Denoting the $(i, t)$-th element of $\widetilde{\bm Q}$ as $\widetilde q_i(x_{u_t})$, we see that this value is equal to $q_i(x_{u_t})$. Note that, unlike the notation used in Section~\ref{sec:model-map-resampled-texts}, $\widetilde{\bm q}$ allows for duplication of the resampled texts and its columns are resampled from the double-centered matrix $\bm Q$.

Let $\bm w_n = (1/n\pi_{u_1}, \ldots, 1/n\pi_{u_n})^\top\in\mathbb{R}^n$ be the weights on the resampled texts, and let $\bm W = \mathrm{diag}(\bm w_n)$ be the corresponding diagonal matrix. Define $g_{ij} = \|\bm q_i - \bm q_j\|^2$ and $\widetilde g_{ij} = \|\widetilde{\bm q}_i - \widetilde{\bm q}_j\|_{\bm w_n}$, where the weighted norm is taken with respect to $\bm w_n$.

\paragraph{LS Sampling.}
According to \citet{drineas-kannan-2001}, the expected Frobenius norm of the approximation error $\bbE\left[\| \widetilde{\bm Q} \bm W \widetilde{\bm Q}^\top - \bm Q \bm Q^\top \|_F^2\right]$ is minimized when the resampling probabilities satisfy $\pi_s \propto \|Q^{(s)}\|^2$.

\paragraph{KL sampling (Proposed).}
Instead of approximating the inner products in $\bm Q$, we aim to approximate the sum of the pairwise distances. We prove that the resampling probabilities $\pi_s$ that minimize $\bbE\left[\sum_{i,j=1}^K (\widetilde g_{ij} - g_{ij})^2\right]$ are given by
\[
\pi_s \propto \sqrt{\sum_{i,j=1}^K \left(q_i(x_s) - q_j(x_s)\right)^4}.
\]

\begin{lemma}\label{lem:unbiased}
For any $i, j \in \{1, \ldots, K\}$, it holds that
    $$\mathbb{E}\left[\widetilde{g}_{ij}\right] = g_{ij}.$$
\end{lemma}

\begin{proof}
    \begin{align*}
        & \bbE[\widetilde g_{ij}] \\
        &= \bbE[\|\widetilde{\bm q}_i - \widetilde{\bm q}_j\|^2_{\bm w_n}]\\
        & = \bbE\left[\sum_{t=1}^n \dfrac{1}{n\pi_{u_t}}(\widetilde q_i(x_{u_t}) - \widetilde q_j(x_{u_t}))^2 \right]\\
        &=\dfrac{1}{n}\sum_{t=1}^n \mathbb{E}\left[\dfrac{1}{\pi_{u_t}}( q_i(x_{u_t}) -  q_j(x_{u_t}))^2\right]\\
        &= \dfrac{1}{n}\sum_{t=1}^n\sum_{s=1}^N \dfrac{1}{\pi_s} (q_i(x_s)-q_j(x_s))^2 \pi_s\\
        &=\|\bmq_i - \bmq_j\|^2\\
        &= g_{ij}.
    \end{align*}
\end{proof}

\begin{lemma}\label{lem:variance}
The variance of the weighted distance after resampling is given by
    \begin{align*}
    & \Var\left(\widetilde{g}_{ij}\right) \\
    & = \dfrac{1}{n}\sum_{s=1}^N\dfrac{1}{\pi_s}(q_i(x_s) - q_j(x_s))^4 - \dfrac{1}{n}\|\bmq_{i} - \bmq_{j}\|^4.
    \end{align*}
\end{lemma}
\begin{proof}

Noting that the resampling is performed independently, we obtain
    \begin{align*}
        & \Var(\widetilde{g}_{ij}) \\
        &= \Var\left(\sum_{t=1}^n \dfrac{1}{n\pi_{u_t}}(\widetilde q_i(x_{u_t}) - \widetilde q_j(x_{u_t})^2\right)\\
        &= \sum_{t=1}^n \Var\left(\dfrac{1}{n\pi_{u_t}}(\widetilde q_i(x_{u_t}) - \widetilde q_j(x_{u_t})^2\right)\\
        &=\sum_{t=1}^n \left(\dfrac{1}{n^2}\sum_{s=1}^N\dfrac{1}{\pi_s}(q_i(x_s) - q_j(x_s))^4 \right.
        \\ &\qquad -\left. \dfrac{1}{n^2}\|\bmq_{i} - \bmq_{j}\|^4\right)\\
        &=\dfrac{1}{n}\sum_{s=1}^N\dfrac{1}{\pi_s}(q_i(x_s) - q_j(x_s))^4 - \dfrac{1}{n}\|\bmq_{i} - \bmq_{j}\|^4.
    \end{align*}
\end{proof}

\begin{proposition}
    The resampling probabilites $\pi_s$ that minimize $\bbE\left[\sum_{i,j = 1}^K (\widetilde{g}_{ij} - g_{ij})^2\right]$ are given by
    \[\pi_s\propto \sqrt{\sum_{i,j=1}^K \left(q_i(x_s) - q_j(x_s)\right)^4}.\]
\end{proposition}
\begin{proof}
By Lemma~\ref{lem:unbiased}, we have
\[
   \bbE\left[\sum_{i,j=1}^K (\widetilde{g}_{ij} - g_{ij})^2\right] = \sum_{i,j = 1}^K \Var(\widetilde{g}_{ij}).
\]
Therefore, by Lemma~\ref{lem:variance}, we aim to minimize
\[
f(\pi_1,\ldots, \pi_N) = \sum_{i,j=1}^K \sum_{s=1}^N \frac{1}{\pi_s} \left(q_i(x_s) - q_j(x_s)\right)^4
\]
subject to the constraint $\sum_{s=1}^N \pi_s = 1$.
Let
\begin{align*}
&g(\pi_1, \ldots, \pi_N, \lambda) \\
&= f(\pi_1, \ldots, \pi_N) + \lambda \left(\sum_{s=1}^N \pi_s - 1\right).
\end{align*}
By setting $\dfrac{\partial g}{\partial \pi_s} = 0$, we obtain
\[
\lambda \pi_s^2 = \sum_{i,j=1}^K \left(q_i(x_s) - q_j(x_s)\right)^4.
\]
Using the condition $\sum_{s=1}^N \pi_s = 1$, it follows that
\[
\sqrt{\lambda} = \sum_{s=1}^N \sqrt{\sum_{i,j=1}^K \left(q_i(x_s) - q_j(x_s)\right)^4}.
\]
Therefore, we have
\[
\pi_s = \frac{\sqrt{\sum_{i,j=1}^K \left(q_i(x_s) - q_j(x_s)\right)^4}}{\sum_{s'=1}^N \sqrt{\sum_{i,j=1}^K \left(q_i(x_{s'}) - q_j(x_{s'})\right)^4}}.
\]
\end{proof}

\section{Error Evaluation without Normalization}
\label{sec:absolute-error}
\begin{figure}
    \centering
    \includegraphics[width=\linewidth]{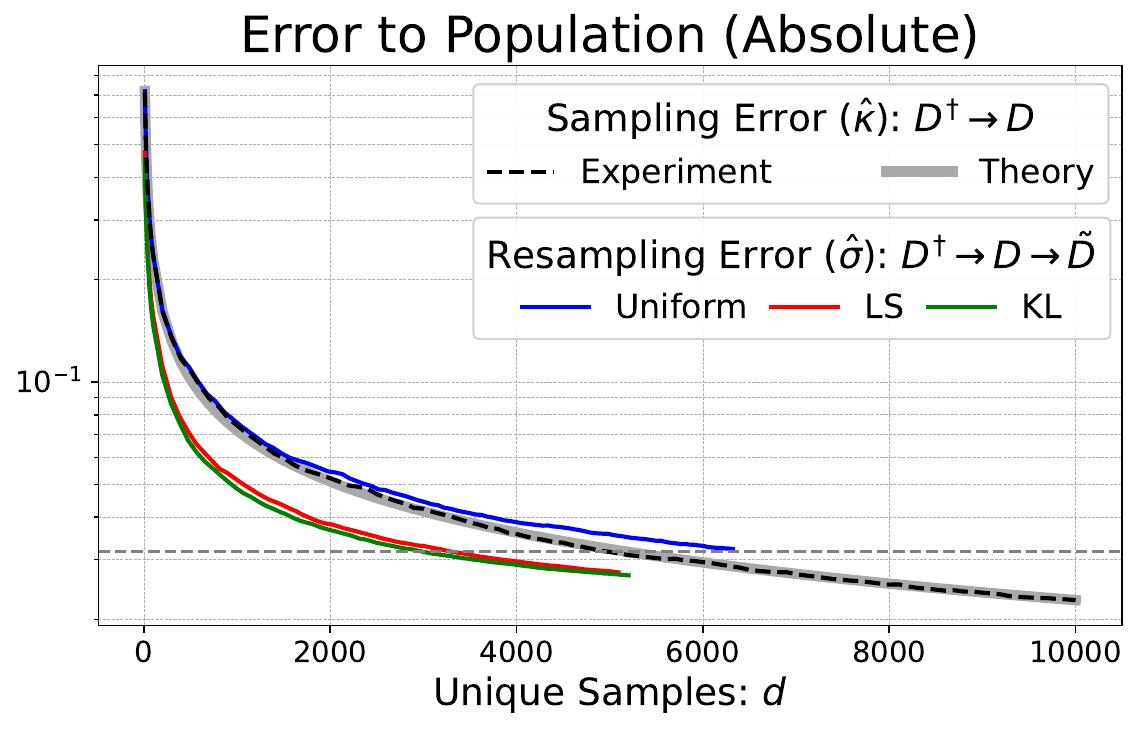}
    \caption{Result of error evaluation without normalization for each model pair. Other settings are the same as Fig.~\ref{fig:diff-ratio-distance}.}
    \label{fig:abs-error-pop}
\end{figure}

In Section~\ref{sec:resampling-error}, the error was normalized by the magnitude of KL divergence for each model pair to evaluate the relative error with respect to KL divergence.
In this section, we evaluate the error without normalization for each model pair. 
The errors $\widetilde e_{ij}$ and $e_{ij}^\dagger$ are redefined by focusing on the sum of squared differences.
\[
\widetilde e_{ij} = \frac{\widetilde{g}_{ij} - g_{ij}}{C},\quad
e_{ij}^\dagger = \frac{g_{ij} - g_{ij}^\dagger}{C^\dagger},
\]
where
\[
C = \frac{1}{K^2}\sum_{i,j=1}^K g_{ij},\quad
C^\dagger = \frac{1}{K^2}\sum_{i,j=1}^K g_{ij}^\dagger.
\]
In the KL sampling method derived in Appendix~\ref{sec:kl-sampling}, optimization is performed to reduce $\bbE[\sum_{i,j=1}^K (\widetilde{g}_{ij} - g_{ij})^2 \mid D_N]$, which corresponds to minimizing the sum of absolute squared errors.

The results of evaluating the error are shown in Fig.~\ref{fig:abs-error-pop}.
KL sampling indeed has a slightly smaller error than LS sampling under this evaluation.
On the other hand, in Fig.~\ref{fig:diff-ratio-distance}, LS sampling can have a slightly smaller error than KL sampling.
In both plots, the performance difference between LS sampling and KL sampling is very small. 
Therefore, this paper primarily explains the simpler LS sampling method due to its comparable performance.

\section{Estimating the Population Error of Resampled Distances} \label{sec:resampling-error-to-population}
This section provides a detailed discussion of the error evaluation method introduced in Section~\ref{sec:resampling-error}.

\paragraph{Sampling Error.}
We consider the case where a dataset $D_n = (x_1, \ldots, x_n)$ of size $n$ is sampled from a large population of texts $D^\dagger = (x^\dagger_1, \ldots, x^\dagger_{N_0})$. We assume that $N_0 \gg n$ and $n$ is sufficiently large. Let $\bmL^\dagger \in \mathbb{R}^{K \times N_0}$ and $\bmL \in \mathbb{R}^{K \times n}$ denote the log-likelihood matrices evaluated on $D^\dagger$ and $D_n$, respectively. Their doubly centered versions are denoted by $\bmQ^\dagger = [\bmQ^{\dagger(1)}, \ldots, \bmQ^{\dagger(N_0)}]$ and $\bmQ = [\bmQ^{(1)}, \ldots, \bmQ^{(n)}]$, with row vectors $\bmq_i^\dagger \in \mathbb{R}^{N_0}$ and $\bmq_i \in \mathbb{R}^n$ representing the coordinates of model $p_i$ in each case.
When $n$ is sufficiently large, the average vector used in centering $\bmQ$ can be well approximated by that of $\bmQ^\dagger$, and thus each column of $\bmQ$ may be regarded as a random sample (without replacement) from the columns of $\bmQ^\dagger$.

We focus on the squared Euclidean distance between two models $p_i$ and $p_j$. 
To ensure comparability with the case where the data size is $n = N$, 
we introduce a scaling factor to normalize the estimate and define:
\begin{align*}
g_{ij} &= \frac{N}{n} \|\bmq_i - \bmq_j \|^2 \\
       &= \frac{N}{n} \sum_{s=1}^n \left(q_i(x_s) - q_j(x_s)\right)^2, \\
g_{ij}^\dagger &= \frac{N}{N_0} \|\bmq_i^\dagger - \bmq_j^\dagger \|^2 \\
       &= \frac{N}{N_0} \sum_{s=1}^{N_0} \left(q_i^\dagger(x^\dagger_s) - q_j^\dagger(x^\dagger_s)\right)^2.
\end{align*}
We define the sampling error as
\[
\varepsilon_{ij} = g_{ij} - g_{ij}^\dagger.
\]
Since the terms $\{(q_i(x_s) - q_j(x_s))^2\}_{s=1}^n$ can be regarded as random samples from $\{(q_i^\dagger(x^\dagger_s) - q_j^\dagger(x^\dagger_s))^2\}_{s=1}^{N_0}$, 
$g_{ij}$ is an unbiased estimator of $g_{ij}^\dagger$, and thus the sampling error satisfies $\mathbb{E}[\varepsilon_{ij}] = 0$.

The mean squared error (MSE) of this sampling error is then given by
\[
\kappa_{ij,n}^2 = \mathbb{E}[\varepsilon_{ij}^2].
\]
Following the formulation in Section~\ref{sec:resampling-error}, we define the aggregated sampling MSE as
\[
\kappa_n^2 = \frac{1}{K^2} \sum_{i,j=1}^K \frac{\kappa_{ij,n}^2}{\max(g_{ij}^\dagger, \varepsilon_0)^2},
\]
where $\varepsilon_0 > 0$ is a small constant to avoid division by zero.

\paragraph{Bootstrap Estimate of Sampling Error.}
We consider a bootstrap procedure that randomly samples $n$ texts uniformly with replacement from the dataset $D_n$ (we note that $D_n$ may later be replaced by $D_N$). That is, each text in $D_n$ is selected with equal probability $\pi_s = 1/n$.
Let $\widetilde D_n = (\widetilde x_1, \ldots, \widetilde x_n)$ denote the resampled dataset. The log-likelihood matrix and its doubly centered version for $\widetilde D_n$ are denoted by $\widetilde{\bmL} \in \mathbb{R}^{K \times n}$ and $\widetilde{\bmQ} \in \mathbb{R}^{K \times n} = [\widetilde{\bmQ}^{(1)}, \ldots, \widetilde{\bmQ}^{(n)}]$, respectively, and the coordinate vector of model $p_i$ is denoted by $\widetilde{\bmq}_i \in \mathbb{R}^n$.
When $n$ is sufficiently large, the mean vector used for centering $\widetilde{\bmQ}$ can be well approximated by that of $\bmQ$, so each column of $\widetilde{\bmQ}$ may be regarded as a random sample (with replacement) from the columns of $\bmQ$.

Unlike in Section~\ref{sec:model-map-resampled-texts}, where resampling with replacement was handled by recording the number of duplicates and incorporating them as weights, 
we here represent the resampled data explicitly by allowing duplicate entries in the coordinate vectors $\widetilde{\bmq}_i$. 
Thus, the difference is only notational, and both approaches describe the same underlying resampling process.

As in the previous subsection, we focus on the squared Euclidean distance between models $p_i$ and $p_j$. To ensure comparability with the case of dataset size $n = N$, we scale the estimate as follows:
\begin{align*}
\widetilde g_{ij} &= \frac{N}{n} \|\widetilde{\bmq}_i - \widetilde{\bmq}_j \|^2 \\
  &= \frac{N}{n} \sum_{s=1}^n \left( \widetilde q_i(\widetilde x_s) - \widetilde q_j(\widetilde x_s) \right)^2.
\end{align*}
The terms $\{ (\widetilde q_i(\widetilde x_s) - \widetilde q_j(\widetilde x_s))^2 \}_{s=1}^n$ can be viewed as a bootstrap sample drawn (with replacement) from the set $\{ (q_i(x_s) - q_j(x_s))^2 \}_{s=1}^n$.
We define the resampling error as
\[
\widetilde\varepsilon_{ij} = \widetilde g_{ij} - g_{ij},
\]
which satisfies $\mathbb{E}[\widetilde\varepsilon_{ij} \mid D_n] = 0$ by construction.
The conditional mean squared error (MSE) of $\widetilde g_{ij}$ given $D_n$ is defined as
\[
\tau^2_{ij, n} = \mathbb{E}\left[\widetilde\varepsilon_{ij}^2 \mid D_n \right].
\]
In practice, we estimate this quantity by performing $R$ independent bootstrap trials and computing
\[
\tau^2_{ij, n} = \frac{1}{R} \sum_{r=1}^R \left( \widetilde\varepsilon_{ij}^{(r)} \right)^2,
\]
where $\widetilde\varepsilon_{ij}^{(r)}$ is the resampling error from the $r$-th trial.

This quantity $\tau^2_{ij,n}$ serves as the bootstrap estimate of $\kappa^2_{ij,n}$~\citep{efron-tibshirani-bootstrap}:
\[
\hat \kappa^2_{ij,n} = \tau^2_{ij,n}.
\]
We may also replace $D_n$ with $D_N$; in that case, this procedure corresponds to an $n$-out-of-$N$ bootstrap~\citep{bickel2008choice}, and $\tau^2_{ij,n}$ remains a valid estimator of $\kappa^2_{ij,n}$~\citep{shimodaira2014higher}.
Moreover, the resampling MSE obeys a standard scaling law $\tau^2_{ij,n} \propto n^{-1}$, and can be approximated by the theoretical relation
\[
\tau^2_{ij,n} = \frac{N}{n} \tau^2_{ij,N}.
\]

Following Section~\ref{sec:resampling-error}, we define the aggregated resampling MSE as
\[
\tau^2_{\text{unif}, n} = \frac{1}{K^2} \sum_{i,j=1}^K \frac{\tau^2_{ij,n}}{\max(g_{ij}, \varepsilon_0)^2}.
\]
This corresponds to the uniform resampling case discussed in Section~\ref{sec:resampling-error}.
The above discussion on $\tau^2_{ij,n}$ also applies directly to $\tau^2_{\text{unif}, n}$, and we obtain the bootstrap estimate of the sampling MSE $\kappa^2_n$ as
\[
\hat \kappa^2_n = \tau^2_{\text{unif}, n}.
\]
Theoretical scaling gives the approximation
\[
\tau^2_{\text{unif}, n} = \frac{N}{n} \tau^2_{\text{unif}, N}.
\]

\paragraph{Decomposition of Population Error.}
We consider a two-stage sampling procedure: first, we obtain a dataset $D_N$ by randomly sampling $N$ texts from the population $D^\dagger$; then, we perform weighted resampling with replacement from $D_N$ to obtain a smaller dataset $\widetilde D_n$ of size $n$. As described in Section~\ref{sec:N-to-d}, we consider three types of resampling weights: LS sampling, KL sampling, and uniform sampling.

The estimated squared distance based on the resampled data is defined as
\[
\widetilde g_{ij} = \| \widetilde{\bmq}_i - \widetilde{\bmq}_j \|_{\bm{w}_d}^2,
\]
where the notation already includes the scaling factor $N/n$ through the weights. In particular, for uniform sampling, we have $\pi_s = 1/N$, which leads to the weight $1/(n \pi_s) = N/n$.

We analyze the error of the resampling estimator $\widetilde g_{ij}$ relative to the true value $g_{ij}^\dagger$ in the population. Define the error as
\[
\varepsilon_{ij}^\dagger = \widetilde g_{ij} - g_{ij}^\dagger.
\]
This error can be decomposed as
\begin{align*}
  \varepsilon_{ij}^\dagger &=
  (\widetilde g_{ij} - g_{ij}) + (g_{ij} - g_{ij}^\dagger) \\
  &= \widetilde\varepsilon_{ij} + \varepsilon_{ij},
\end{align*}
where $\widetilde\varepsilon_{ij}$ and $\varepsilon_{ij}$ denote the resampling error and the sampling error, respectively.

Taking the expectation of the squared error, we obtain the decomposition of the MSE:
\begin{align*}
\sigma_{ij}^2 &= \mathbb{E}[ (\varepsilon_{ij}^\dagger)^2 ] \\
&= \mathbb{E}[ (\widetilde\varepsilon_{ij} + \varepsilon_{ij})^2 ] \\
&= \mathbb{E}[\mathbb{E}[ \widetilde\varepsilon_{ij}^2 \mid D_N]]\\
&\qquad
+ 2 \mathbb{E}[\mathbb{E}[ \widetilde\varepsilon_{ij} \mid D_N] \varepsilon_{ij}]
+ \mathbb{E}[\varepsilon_{ij}^2 ] \\
&= \mathbb{E}[\tau_{ij,n}^2] + \kappa_{ij,N}^2,
\end{align*}
where we used $\mathbb{E}[\widetilde\varepsilon_{ij} \mid D_N] = 0$.

Here, $\tau_{ij,n}^2$ is computed from resampling errors $\widetilde\varepsilon_{ij}^{(r)}$ obtained via weighted resampling from $D_N$. The term $\kappa_{ij,N}^2$ can be estimated by the bootstrap MSE under uniform resampling of size $N$, denoted by $\tau_{ij,\text{unif}, N}^2$. Thus, we estimate $\sigma_{ij}^2$ by
\[
\hat\sigma_{ij}^2 = \tau_{ij,n}^2 + \tau_{ij,\text{unif}, N}^2.
\]

As in Section~\ref{sec:resampling-error}, we define the aggregated population MSE by
\[
\sigma_n^2 = \frac{1}{K^2} \sum_{i,j=1}^K \frac{\sigma_{ij,n}^2}{\max(g_{ij}^\dagger, \varepsilon_0)^2}.
\]
Substituting the estimate $\hat\sigma_{ij}^2$ into this expression yields the estimate of $\sigma_n^2$:
\[
\hat\sigma_n^2 = \tau_n^2 + \tau_{\text{unif}, N}^2.
\]

\section{Model Performance Prediction}\label{sec:predict-performance}
\begin{figure*}[!t]
  \centering
  \includegraphics[width=\linewidth]{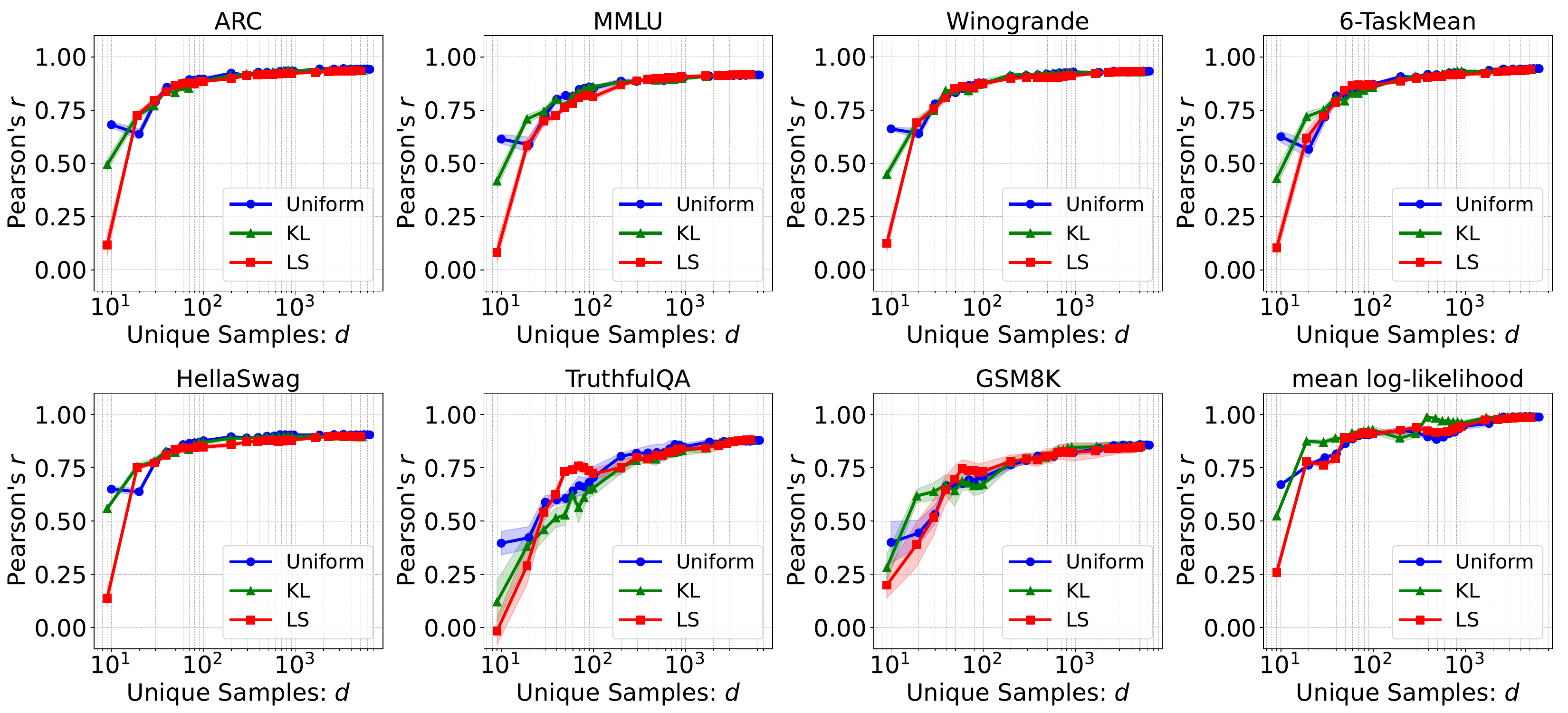}
  \caption{Pearson's correlation coefficient ($r$) between the predicted scores and the benchmark scores as a function of the number of unique texts $d$ (determined by the resampling size $n$), plotted separately for each resampling method.
Solid lines indicate the mean across five different data splits, and the shaded bands show $\pm$1 standard deviation. 
For every task and every method, predictive performance improves as $d$ increases.}
  \label{fig:pred_results_pearson}
\end{figure*}

\begin{table*}[t]
\tiny
\centering
\begin{tabular}{@{\hspace{1.2em}}r@{\hspace{1.2em}}r@{\hspace{1.2em}}l@{\hspace{1.2em}}r@{\hspace{1.2em}}r@{\hspace{1.2em}}r@{\hspace{1.2em}}r@{\hspace{1.2em}}r@{\hspace{1.2em}}r@{\hspace{1.2em}}r@{\hspace{1.2em}}r@{\hspace{1.2em}}}
\toprule
$n$ & $d$ & Method & ARC & HellaSwag & MMLU & TruthfulQA & Winogrande & GSM8K & 6-TaskMean & mean log-likelihood \\
\midrule
\multirow{3}{*}{$10^1$} & 10 & Uniform & 0.682 $\pm$ 0.014 & 0.650 $\pm$ 0.014 & 0.614 $\pm$ 0.022 & 0.396 $\pm$ 0.056 & 0.662 $\pm$ 0.008 & 0.400 $\pm$ 0.099 & 0.625 $\pm$ 0.026 & 0.671 $\pm$ 0.007 \\
 & 9 & KL & 0.494 $\pm$ 0.029 & 0.559 $\pm$ 0.010 & 0.417 $\pm$ 0.034 & 0.120 $\pm$ 0.101 & 0.448 $\pm$ 0.025 & 0.281 $\pm$ 0.063 & 0.429 $\pm$ 0.041 & 0.524 $\pm$ 0.012 \\
 & 9 & LS & 0.116 $\pm$ 0.051 & 0.138 $\pm$ 0.036 & 0.082 $\pm$ 0.045 & -0.016 $\pm$ 0.065 & 0.124 $\pm$ 0.040 & 0.199 $\pm$ 0.061 & 0.104 $\pm$ 0.043 & 0.258 $\pm$ 0.030 \\
\midrule
\multirow{3}{*}{$10^2$} & 100 & Uniform & 0.896 $\pm$ 0.012 & 0.877 $\pm$ 0.007 & 0.852 $\pm$ 0.012 & 0.705 $\pm$ 0.052 & 0.878 $\pm$ 0.007 & 0.706 $\pm$ 0.034 & 0.872 $\pm$ 0.006 & 0.909 $\pm$ 0.010 \\
 & 99 & KL & 0.888 $\pm$ 0.012 & 0.870 $\pm$ 0.015 & 0.861 $\pm$ 0.006 & 0.655 $\pm$ 0.036 & 0.875 $\pm$ 0.021 & 0.674 $\pm$ 0.043 & 0.857 $\pm$ 0.008 & 0.928 $\pm$ 0.018 \\
 & 99 & LS & 0.885 $\pm$ 0.005 & 0.847 $\pm$ 0.008 & 0.812 $\pm$ 0.016 & 0.722 $\pm$ 0.023 & 0.874 $\pm$ 0.009 & 0.733 $\pm$ 0.039 & 0.868 $\pm$ 0.006 & 0.910 $\pm$ 0.012 \\
\midrule
\multirow{3}{*}{$10^3$} & 949 & Uniform & 0.933 $\pm$ 0.004 & 0.905 $\pm$ 0.005 & 0.903 $\pm$ 0.007 & 0.849 $\pm$ 0.015 & 0.929 $\pm$ 0.004 & 0.821 $\pm$ 0.012 & 0.925 $\pm$ 0.005 & 0.946 $\pm$ 0.014 \\
 & 912 & KL & 0.937 $\pm$ 0.003 & 0.898 $\pm$ 0.006 & 0.899 $\pm$ 0.010 & 0.832 $\pm$ 0.025 & 0.926 $\pm$ 0.002 & 0.847 $\pm$ 0.022 & 0.933 $\pm$ 0.007 & 0.960 $\pm$ 0.012 \\
 & 898 & LS & 0.923 $\pm$ 0.003 & 0.880 $\pm$ 0.004 & 0.906 $\pm$ 0.007 & 0.840 $\pm$ 0.012 & 0.911 $\pm$ 0.007 & 0.822 $\pm$ 0.041 & 0.919 $\pm$ 0.009 & 0.943 $\pm$ 0.015 \\
\midrule
\multirow{3}{*}{$10^4$} & 6335 & Uniform & 0.942 $\pm$ 0.002 & 0.905 $\pm$ 0.004 & 0.916 $\pm$ 0.006 & 0.879 $\pm$ 0.018 & 0.933 $\pm$ 0.005 & 0.857 $\pm$ 0.018 & 0.945 $\pm$ 0.004 & 0.989 $\pm$ 0.006 \\
 & 5240 & KL & 0.937 $\pm$ 0.002 & 0.896 $\pm$ 0.005 & 0.917 $\pm$ 0.005 & 0.885 $\pm$ 0.009 & 0.930 $\pm$ 0.003 & 0.857 $\pm$ 0.026 & 0.941 $\pm$ 0.005 & 0.988 $\pm$ 0.006 \\
 & 5080 & LS & 0.935 $\pm$ 0.002 & 0.897 $\pm$ 0.006 & 0.918 $\pm$ 0.007 & 0.882 $\pm$ 0.014 & 0.931 $\pm$ 0.005 & 0.851 $\pm$ 0.019 & 0.941 $\pm$ 0.004 & 0.986 $\pm$ 0.007 \\
\bottomrule
\end{tabular}

    \caption{Summary of the representative values from Fig.~\ref{fig:pred_results_pearson}.
For each resampling method, and for $n = 10^{1}, 10^{2}, 10^{3}, 10^{4}$ (with the corresponding numbers of unique texts $d$), the table reports Pearson's correlation $r$ between the predicted and true benchmark scores, together with $\pm$1 standard deviation.}
    \label{tab:pred_results_pearson}
\end{table*}

We computed the log-likelihoods of the unique texts contained in the resampled data and, following~\citet{modelmap2025}, used these values to predict model performance.  
This section describes the experiments in detail.

\subsection{Model Performance}
Following~\citet{modelmap2025}, we used the Open LLM Leaderboard v1~\cite{open-llm-leaderboard-v1} as the source of model performance scores\footnote{\url{https://huggingface.co/spaces/open-llm-leaderboard-old/open_llm_leaderboard}}.
The leaderboard provides scores for the following six benchmark tasks\footnote{Benchmark scores are available for 996 of the 1,018 models released by~\citet{modelmap2025}. For convenience, we denote this subset size by $K$ throughout this section.}: AI2 Reasoning Challenge (ARC)~\cite{arxiv:1803.05457}, HellaSwag~\cite{arxiv:1905.07830}, MMLU~\cite{arxiv:2009.03300}, TruthfulQA~\cite{arxiv:2109.07958}, Winogrande~\cite{arxiv:1907.10641}, and GSM8K~\cite{arxiv:2110.14168}.

In addition to these benchmark scores, we followed~\citet{modelmap2025} and also predicted (i) the average across the six tasks (hereafter referred to as 6-TaskMean) and (ii) the mean log-likelihood $\bar\ell_i$ of the log-likelihood vector $\boldsymbol\ell_i\in\mathbb{R}^N$.

\subsection{Dataset Configuration}
The number of resampled texts was set to
\begin{align*}
n \in \{&10,20,\ldots,90,100,200,\ldots,900,\\
        &1000,2000,\ldots,9000,10000\}.
\end{align*}
As explained in Section~\ref{sec:N-to-d}, resampling $n$ texts from $D_N$ yields $\widetilde{D}_n$, and from this resampled set we extract a set of $d$ unique texts, $D^{\ast}_d=\{x_{u_1},\dots,x_{u_d}\}$.
Using these $d$ texts, we computed the log-likelihood matrix $\boldsymbol{L}_d\in\mathbb{R}^{K\times d}$ and then formed the doubly-centered matrix $\widetilde{\boldsymbol{Q}}_d=[\widetilde{\boldsymbol{q}}_{1},\dots,\widetilde{\boldsymbol{q}}_{K}]^{\top}\in\mathbb{R}^{K\times d}$ with scaling weights $\boldsymbol{w}_d=\left(\frac{c(u_1)}{n\pi_{u_1}},\dots,\frac{c(u_d)}{n\pi_{u_d}}\right)^{\top}\!\in\mathbb{R}^d$.

For each benchmark task, the dataset is given as $\{(\widetilde{\boldsymbol{q}}_{1},f_1),\dots,(\widetilde{\boldsymbol{q}}_{K},f_K)\}$,  
where $\widetilde{\boldsymbol{q}}_{i}\!=\!(\widetilde{q}_{i}(x_{u_1}),\dots,\widetilde{q}_{i}(x_{u_d}))^{\top}\in\mathbb{R}^d$ is the $i$-th row of $\widetilde{\boldsymbol{Q}}_d$ corresponding to language model $p_i$, and $f_i\in[0,100]$ is its benchmark score.

\subsection{Regression Formulation}
As in~\citet{modelmap2025}, we adopted ridge regression to predict each benchmark score.
The matrix of explanatory variables is
\begin{align}
\widetilde{\boldsymbol{Q}}_d \boldsymbol{W}_d^{1/2}\in\mathbb{R}^{K\times d},\label{eq:ridge_explanatory}
\end{align}
where the diagonal matrix $\boldsymbol{W}_d^{1/2}\in\mathbb{R}^{d\times d}$ has $\sqrt{\frac{c(u_t)}{n\pi_{u_t}}}$ on its $t$-th diagonal entry\footnote{
We adopt $\widetilde{\boldsymbol{Q}}_d\boldsymbol{W}_d^{1/2}$ as the matrix of explanatory variables rather than $\widetilde{\boldsymbol{Q}}_d$ itself.
Let $\boldsymbol{W}_d=\mathrm{diag}(\boldsymbol{w}_d)$, whose $t$-th diagonal entry is $\frac{c(u_t)}{n\pi_{u_t}}$.  
Since $\boldsymbol{W}_d^{1/2}\boldsymbol{W}_d^{1/2}=\boldsymbol{W}_d$, we have $\widetilde{\boldsymbol{Q}}_d\boldsymbol{W}_d^{1/2}(\widetilde{\boldsymbol{Q}}_d\boldsymbol{W}_d^{1/2})^{\top}=\widetilde{\boldsymbol{Q}}_d\boldsymbol{W}_d\widetilde{\boldsymbol{Q}}_d^{\top}$.
Then Lemma 1 of~\citet{drineas-kannan-2001} gives
$\mathbb{E}\bigl[\widetilde{\boldsymbol{Q}}_d\boldsymbol{W}_d\widetilde{\boldsymbol{Q}}_d^{\top}\bigr]=\boldsymbol{Q}\boldsymbol{Q}^{\top}$.
Thus, pre-multiplying by $\boldsymbol{W}_d^{1/2}$ preserves this desirable expectation while appropriately re-scaling the features.
}.
Let $\boldsymbol{f}=(f_1,\dots,f_K)^{\top}\in\mathbb{R}^{K}$ denote the vector of target variables.
The objective function, parameterized by $\boldsymbol{\theta}\in\mathbb{R}^{d}$, is defined as
\[
\mathcal{L}(\boldsymbol{\theta})
= \|\boldsymbol{f}-\widetilde{\boldsymbol{Q}}_d\boldsymbol{W}_d^{1/2}\boldsymbol{\theta}\|^2
  + \alpha\|\boldsymbol{\theta}\|^2,
\]
where $\alpha\in\mathbb{R}_{>0}$ is a hyperparameter controlling the strength of regularization.

\subsection{Training Setup}
We partitioned the set of models into five folds according to their model types, as defined in~\citet{modelmap2025}.
We then trained the parameters and predicted the benchmark scores.
Training was performed with~\texttt{RidgeCV} from scikit-learn~\cite{DBLP:journals/sigmobile/VaroquauxBLGPM15}.
To account for randomness, we repeated the data split with five different random seeds.
As the evaluation metric, we computed Pearson's correlation coefficient ($r$) between the predicted and true benchmark scores for each split and averaged the results.

For each training set (i.e., the four folds in the outer five-fold CV), we conducted an inner five-fold cross-validation to select $\alpha$ from $\{10^{1},\dots,10^{9}\}$, again following~\citet{modelmap2025}.
The predicted scores were then clipped to the range $[0,100]$.  
When the target variable $\boldsymbol{f}$ was the mean log-likelihood $(\bar\ell_1,\dots,\bar\ell_K)\in\mathbb{R}^K$, we searched $\alpha$ in $\{10^{-4},\dots,10^{4}\}$ and did not clip the predictions.

\subsection{Results}
Figure~\ref{fig:pred_results_pearson} shows Pearson's correlation coefficient between the predicted scores and the benchmark scores for Uniform, KL, and LS sampling as a function of the number of unique texts $d$ for each resampling size $n$.
Table~\ref{tab:pred_results_pearson} summarizes representative values obtained for each method at $n=10^{1},10^{2},10^{3},10^{4}$ (and the corresponding $d$).

For all tasks and all methods, Pearson's correlation coefficient $r$ increases as $d$ grows.
As shown in Table~\ref{tab:pred_results_pearson}, even at $d\approx 100$ ---for example, for 6-TaskMean--- the predicted scores already achieve $r \approx 0.85$ under every resampling method, and only minor differences are observed among the strategies.
Hence, predictive performance depends almost solely on the number of unique texts $d$.

This behavior can be interpreted as follows: when $d\ll K$ ($K\!\approx\!10^{3}$), the column vectors of $\widetilde{\boldsymbol{Q}}_d\boldsymbol{W}_d^{1/2}\in\mathbb{R}^{K\times d}$ span a subspace of insufficient dimensionality, limiting the expressive power of the regression model.
As $d$ increases, the feature space expands and ridge regression becomes effective, leading to a rapid improvement in performance; however, once $d\gtrsim10^{3}$ provides sufficient dimensionality, further gains in correlation are gradual.

\end{document}